\documentclass{article} % For LaTeX2e

\usepackage[acronym]{glossaries}
\makeglossaries

\usepackage{amsmath}
\usepackage{amssymb}
\usepackage{amsthm}
\usepackage{dsfont}
\usepackage{mathtools}
\allowdisplaybreaks

\usepackage{graphicx}
\usepackage{wrapfig}

\usepackage{booktabs}
\usepackage{caption}

\usepackage{xcolor}
\usepackage{parskip} % helps with indenting and paragraph layout
\usepackage{soul} % use \ul to underline, fixes wrapping issues

\usepackage{caption}
\usepackage{subcaption}
\usepackage[round]{natbib}
\usepackage{hyperref}
\hypersetup{
  colorlinks=true,
  linkcolor=blue,
  citecolor=blue,
  urlcolor=blue
}
\usepackage{cleveref}

\definecolor{MidnightBlue}{rgb}{0.1, 0.1, 0.44}
\definecolor{mahogany}{rgb}{0.75, 0.25, 0.0}

\usepackage{listings}
\usepackage{fancyvrb}
\fvset{fontsize=\normalsize}

\lstdefinestyle{mystyle}{
    commentstyle=\color{OliveGreen},
    keywordstyle=\color{BurntOrange},
    numberstyle=\tiny\color{black!60},
    stringstyle=\color{MidnightBlue},
    basicstyle=\ttfamily,
    breakatwhitespace=false,
    breaklines=true,
    captionpos=b,
    keepspaces=true,
    numbers=left,
    numbersep=5pt,
    showspaces=false,
    showstringspaces=false,
    showtabs=false,
    tabsize=2
}
\usepackage{tikz}
\usetikzlibrary{shapes,decorations,arrows,calc,arrows.meta,fit,positioning}
\tikzset{
    -Latex,auto,node distance =1 cm and 1 cm,semithick,
    state/.style ={circle, draw, minimum width = 0.7 cm},
    detstate/.style ={rectangle, draw, minimum width = 0.7 cm, minimum height = 0.7 cm},
    point/.style = {circle, draw, inner sep=0.04cm,fill,node contents={}},
    bidirected/.style={Latex-Latex,dashed},
    el/.style = {inner sep=2pt, align=left, sloped}
}
\usepackage[algoruled, algo2e]{algorithm2e}
\usepackage{algorithm}
\usepackage{algorithmic}

\newcommand{\g}{\,|\,}

\newtheorem*{assumption*}{Assumption}

\newtheorem*{corollary*}{Corollary}

\newtheorem*{definition*}{Definition}
\newtheorem{lemma}{Lemma}
\newtheorem*{lemma*}{Lemma}

\newtheorem*{proposition*}{Proposition}
\newtheorem{theorem}{Theorem}
\newtheorem*{theorem*}{Theorem}

\renewcommand{\mid}{~\vert~}
\newcommand{\cL}{\mathcal{L}}
\newcommand{\cN}{\mathcal{N}}

\newcommand{\E}{\mathop{\mathbb{E}}} % need mathop so it works with /limits

\newcommand{\Var}{\mathbb{V}\textrm{ar}}
\newcommand{\Cov}{\mathbb{C}\textrm{ov}}
\usepackage{iclr2026_conference,times}

\usepackage[utf8]{inputenc} % allow utf-8 input
\usepackage[T1]{fontenc}    % use 8-bit T1 fonts
\usepackage{hyperref}       % hyperlinks
\usepackage{url}            % simple URL typesetting
\usepackage{booktabs}       % professional-quality tables
\usepackage{amsfonts}       % blackboard math symbols
\usepackage{nicefrac}       % compact symbols for 1/2, etc.
\usepackage{microtype}      % microtypography
\usepackage{xcolor}         % colors
\usepackage{amsmath}

\usepackage{tcolorbox}
\tcbuselibrary{skins, theorems, breakable}

\newtcolorbox{graybox}{
  enhanced,
  colback=black!4,        % very light gray fill
  colframe=black!4,       % frame same color as fill (invisible frame)
  arc=2pt,                % slightly rounded corners
  boxrule=0pt,            % no border line
  left=6pt, right=6pt, top=6pt, bottom=6pt,
}

\newtcolorbox{bluebox}{
  enhanced,
  colback=blue!5,
  colframe=blue!5,
  boxrule=0pt,
  arc=2pt,
  left=8pt, right=8pt, top=6pt, bottom=6pt,
  breakable,
  before skip=10pt, after skip=10pt,
}

\newtcolorbox{orangebox}{
  enhanced, colback=orange!8, colframe=orange!8,
  boxrule=0pt, arc=2pt,
  left=8pt, right=8pt, top=6pt, bottom=6pt,
  breakable, before skip=10pt, after skip=10pt,
}

\newtcolorbox{pinkbox}{
  enhanced, colback=magenta!5, colframe=magenta!5,
  boxrule=0pt, arc=2pt,
  left=8pt, right=8pt, top=6pt, bottom=6pt,
  breakable, before skip=10pt, after skip=10pt,
}

\newtcolorbox{tealbox}{
  enhanced, colback=teal!8, colframe=teal!8,
  boxrule=0pt, arc=2pt,
  left=8pt, right=8pt, top=6pt, bottom=6pt,
  breakable, before skip=10pt, after skip=10pt,
}

\newtcolorbox{greenbox}{
  enhanced, colback=green!6, colframe=green!6,
  boxrule=0pt, arc=2pt,
  left=8pt, right=8pt, top=6pt, bottom=6pt,
  breakable, before skip=10pt, after skip=10pt,
}

\newtcolorbox{yellowbox}{
  enhanced, colback=yellow!8, colframe=yellow!8,
  boxrule=0pt, arc=2pt,
  left=8pt, right=8pt, top=6pt, bottom=6pt,
  breakable, before skip=10pt, after skip=10pt,
}

\newtcolorbox{lavenderbox}{
  enhanced, colback=violet!5, colframe=violet!5,
  boxrule=0pt, arc=2pt,
  left=8pt, right=8pt, top=6pt, bottom=6pt,
  breakable, before skip=10pt, after skip=10pt,
}

\newcommand{\method}{SGFlow}
\usepackage{pifont}
\newcommand{\cmark}{\ding{51}}%
\newcommand{\xmark}{\ding{55}}%

\usepackage{hyperref}
\usepackage{url}

\title{Flow Map Learning via Nongradient Vector Flow}

\author{Mark Goldstein$^1$\thanks{\texttt{mgoldstein@flatironinstitute.org}},
Anshuk Uppal$^2$, Raghav Singhal$^3$, Aahlad Puli$^3$, \& Rajesh Ranganath$^3$
\\
$^1$Center for Computational Mathematics, Flatiron Institute.\\
$^2$Department of Applied Mathematics and Computer Science, Technical University of Denmark.\\
$^3$Courant Institute, New York University.\\
}

\iclrfinalcopy
\begin{document}

\maketitle

\definecolor{dgreen}{rgb}{0.0, 0.52, 0.34}
\newcommand{\dgreen}[1]{\color{dgreen}{#1}}
\definecolor{dblue}{rgb}{0.0, 0.32, 0.52}
\newcommand{\dblue}[1]{\color{dblue}{#1}}

\begin{abstract}
Diffusion and flow-based models benefit from simple regression losses, but inference incurs significant overhead because sampling requires integration. Consistency models address this by directly learning the flow maps along the ODE trajectory, opening a design space between one-step and many-step approaches. However, existing methods face computational challenges such as requiring model inverses or backpropagation through iterated model calls, and do not always prove that the desired ODE flow map is a solution to the loss. We introduce \method{}, an approach for learning flow maps that bypasses explicit invertibility constraints and expensive differentiation through model iteration. \method{} trains a model to compute both the ODE solutions and the implied velocity from scratch by following non-conservative dynamics with a stationary point at the desired flow map. On the CIFAR image benchmark, no single method attains the best FID at every step count: \method{} attains the best FID at 10 sampling steps and remains competitive with flow matching, Meanflow, and Lagrangian map matching at other step counts, while being the only one with a proven stationary-point guarantee for its stopgrad-based dynamics.
\end{abstract}

\section{Introduction}

Diffusion and flow models \citep{sohl2015deep, ho2020denoising, song2020score, kingma2021variational, albergo2022building, 
singhal2023diffuse,pandey2023complete,bartosh2024neural,singhal2024s, albergo2023stochastic, lipman2022flow, liu2022flow} have improved generation in  domains such as proteins 
\citep{abramson2024accurate}
and images
\citep{peebles2023scalable,esser2403scaling}. Sampling from these models typically requires numerically integrating an ordinary or stochastic differential equation. Numerical integration requires multiple forward passes of a neural network, leading to increased sampling latency and cost. 

To ameliorate this generation cost by changing the training,  recent approaches for consistency modeling and map matching 
\citep{song2023consistency,song2023improved,kim2023consistency,lu2024simplifying,boffi2024flow,boffi2025build}
aim to learn direct mappings from noise to intermediate or final data points along trajectories defined by probability flow ODEs, thereby avoiding costly integration.  However, the methods have their respective complexities. 
For example, flow map matching requires model invertibility, while consistency models need either to map in one step or introduce extra steps that leave the target ODE trajectory.

We introduce \method{} (for StopGrad Flow), an approach that builds on flows and map matching methods
and 
\begin{itemize}
 \item Has true flow map as a unique stationary point 
  \item Does not restrict the class of neural networks used (e.g., to invertible functions)
  \item Does not require auxiliary losses involving invertibility or adversarial optimization
  \item Does not require optimizing through nested calls to the model
  \item Allows for generation along the ODE trajectory with any number of steps
\end{itemize}

Existing methods for learning flow maps fall into a few categories in terms of their challenges; all challenges relate to the idea that a flow map is characterized by certain derivative properties and that losses minimize squared error to make these properties hold. Flow map matching and related methods rely on a fundamental relationship between invertible mappings and ordinary differential equations (ODEs). This relationship typically requires explicitly computing both the forward map (the model being trained) and its inverse during training, complicating training, or requires expensive backpropagation through nested model calls. \cite{boffi2025build} propose stopgrad placements for the flow map matching losses of \cite{boffi2024flow} that bypass this expensive nested differentiation, but do not prove that the stopgrads preserve the stationary point at the true flow map. Meanflow \citep{geng2025mean} does not explicitly enforce the model inverse identities, and avoids backpropagation through forward-mode derivatives altogether, achieving good image generation performance with low step counts; Meanflow is also not shown to have a stationary point at the true flow map.

\method{} avoids the complexity of tracking a model and its inverse
 by exploiting an alternate identity involving only
Jacobian-vector products (JVPs) without inverse functions.
This identity allows us to formulate the objective purely in terms of the forward map, without needing explicit access to its inverse. Since solutions to ODEs naturally produce invertible mappings, the \method{} objective implicitly encourages invertibility without explicitly enforcing it. Thus, at optimality, \method{} yields a continuously differentiable function that precisely integrates the velocity field, directly generating the desired data distribution. 
We summarize the trade-offs among recent methods in \Cref{tab:comp} and in \Cref{sec:flowmap_related_work}.

Experimentally,  for a basic training setup using the same common architecture, we ask how flow matching, Meanflow, \method{}, and Lagrangian map matching
compare in moderate dimensions (CIFAR-10) on unconditional metrics (FID) when decreasing the number of sampling steps.

\begin{table}[t]
    \setlength{\tabcolsep}{3.15pt}
    \begin{graybox}
        \resizebox{\textwidth}{!}{
        \centering
    \begin{tabular}{lccccccc}
    \toprule
    \multicolumn{1}{l}{\textbf{Methods}} &
    \textbf{Multistep} &
    \textbf{Follows ODE} &
    \textbf{Sim. Free} &
    \textbf{Regression Loss} &
    \textbf{Inverse-Free} &
    \textbf{Prove Optimum} &
    \textbf{No Nesting}
    \\ 
    \midrule 
    Consistency Distillation \citep{song2023consistency}
    & \xmark & \cmark & \xmark & \cmark & \cmark & \cmark & \cmark
    \\ 
    \midrule 
    Consistency Training \citep{song2023consistency}
    & \xmark & \cmark & \cmark & \cmark & \cmark & \cmark & \cmark
    \\
    \midrule
    Consistency Trajectory Models \citep{kim2023consistency}
    & \cmark & \cmark & \xmark & \xmark & \cmark & \cmark & \cmark
    \\
    \midrule    
    L-FMM \citep{boffi2024flow}
    & \cmark & \cmark & \cmark & \cmark & \xmark & \cmark & \cmark
    \\ 
    \midrule
    LSD \citep{boffi2025build}
    & \cmark & \cmark & \cmark & \cmark & \cmark & \cmark & \xmark
    \\ 
    \midrule
    ESD \citep{boffi2025build}
    & \cmark & \cmark & \cmark & \cmark & \cmark & \cmark & \cmark
    \\
    \midrule
    PSD \citep{boffi2025build}
    & \cmark & \cmark & \cmark & \cmark & \cmark & \cmark & \xmark
    \\
    \midrule
    Meanflow \citep{geng2025mean}
    & \cmark & \cmark & \cmark & \cmark & \cmark & \xmark & \cmark
    \\
    \midrule 
    \textbf{\method{} (this work)}
    & \cmark & \cmark & \cmark & \cmark & \cmark & \cmark & \cmark
    \\ 
    \bottomrule
    \end{tabular}
    }
        \end{graybox}
    \vspace{0.1cm}
    \caption{\textbf{Comparison to prior works}. 
    We categorize flow map learning (or \textit{consistency modeling}) techniques 
    and our proposed \method{} method according to:
    (1) ability to adjust sampling steps post-training,
    (2) whether they follow the PF-ODE \citep{song2020score},
    (3) whether they allow simulation-free training,
    (4) whether their objectives use regression,
    (5) whether training is free of model inversion,
    (6) whether the true flow map is proven to be optimal or stationary,
    and (7) whether training avoids differentiation through nested model calls.
    See \Cref{sec:flowmap_related_work} for details.
    \label{tab:comp}}
\end{table}

\section{Background}
 Stochastic interpolants
\citep{lipman2022flow,albergo2023stochastic},
and more generally most diffusion and flow methods, hereafter just \textit{flows},
pose generative modeling as transport of a simple base density to a target density.
Interpolants tackle the problem as follows. For $t \in [0,1]$:
\begin{enumerate}
    \item  Choose ($\alpha_t$, 
$\sigma_t$) where $\alpha_0=\sigma_1=1$ and $\alpha_1=\sigma_0=0$.
Commonly, $\alpha_t=1-t$ and $\sigma_t=t$.

\item Define
$X_t = \alpha_t X_0 + \sigma_t X_1$ 
for base density $X_0 \sim q_0$ and data $X_1 \sim q_1$ 
(or vice versa).

\item Learn to produce new samples along the trajectory of densities.

\end{enumerate}
For a function $f$, let $\dot f_t := \frac{d}{dt}f_t$.
Thus $\dot X_t := \dot\alpha_t X_0 + \dot \sigma_t X_1$.
It follows that $X_t$ has density $q_t$ satisfying:
\begin{align}
\label{eq:continuity}
 \partial_t q_t(x) = - \nabla_x \cdot (q_t(x) v(t,x)), 
 \quad \quad 
 v(t,x) := \mathbb{E}[\dot X_t \mid X_t =x],
\end{align}
where $v$ is called the velocity.
The PDE in  \Cref{eq:continuity} is derived in the above works. To accomplish step three, one starts by making the observation that  a density satisfies \Cref{eq:continuity}  if and only if it is the density of the solution to the \textit{probability flow ODE} $dX=vdt$ integrated forward from $X_0 \sim q_0$ or in reverse from $X_1 \sim q_1$ \citep{albergo2024learning}.
One then proceeds by first approximating $v$ using the following (simulation-free) loss:
\begin{align}
    \label{eq:flowmatchingloss}
    \mathcal{L}_{v}(v_\theta) = \mathbb{E}\Big[\| v_\theta(t, X_t) - (\dot \alpha_t X_0 + \dot \sigma_t X_1)\|^2 \Big]_{X_t = \alpha_t X_0 + \sigma_t X_1},
\end{align} 
which has minimizer $v_\theta=v$ 
and then solving $dx=v_\theta dt$.

\paragraph{Background on Consistency Methods.}
Sampling from flows requires integration, where each step evaluates a neural network $v_\theta$ modeling a score, velocity, or similar. Knowing integrals of $v$ directly could, in principle, speed up sampling. The goal of consistency and map matching methods is to learn to map along the trajectory implied by the optimal $v$. We review an example here, with others in \Cref{sec:flowmap_related_work}. \cite{song2023consistency, song2023improved} seek to learn a mapping $\hat{g}$ that takes interpolant samples $X_t \sim q_t$ to $\widehat{X}_0$, the $t=0$ solution to $dx = v\,dt$ starting at $X_t$ (note that $\widehat{X}_0$ usually differs from the endpoint $X_0$ used to draw $X_t$). The loss measures the distance between modeled outputs at two nearby points. Let $\text{SG}[\hat g]$ indicate stopgrad. Then:
\begin{align}
    \text{Consistency}(\hat{g}) := \mathbb{E}_{q(X_t)}
    [
        \text{dist}(
        \hat{g}(t, X_t),
        \text{SG}[\hat{g}](t-\Delta t, \widehat{X}_{t - \Delta t})
        )
    ].
\end{align}
The target $\widehat{X}_{t - \Delta t}$ should come from integrating the true velocity a small step $\Delta t$ from $X_t$, but since $v$ is unknown, it is typically approximated using a pretrained $v_\theta$ or a $v_\theta$ derived jointly with $\hat g$ — increasing training cost and introducing approximation error. Allowing multistep sampling requires a re-noising step that takes the trajectory off the probability-flow ODE, so the resulting updates no longer correspond to integrating the PF-ODE. \cite{kim2023consistency} observe that this multistep approach ``exhibits degrading sample quality with increasing NFE, lacking a clear trade-off between computational budget (NFE) and sample fidelity''. Subsequent works introduce various training- and inference-time modifications to bridge the gap between one-step and many-step sampling \citep{song2023consistency, lu2024simplifying, kim2023consistency, boffi2024flow, sabour2025align, geng2025mean, zhou2025inductive}; see \Cref{sec:flowmap_related_work}.

\section{Method}

We present
\method{}, a method for learning to solve the probability flow ODE without adversarial training, without model inverse during training, without representing explicit derivative matrices, and without costly simulations from pretrained models.
\method{} trains a model to compute both the ODE solutions and the implied velocity from scratch by following non-conservative dynamics.

Consider a two-time map $f$ that for $t \leq u$ brings $X_t$ up to $X_u$ by solving the probability flow ODE $dx=vdt$.
Such an $f$ that integrates $v$ can be defined as follows:
\begin{align}
    f(t,u,x) &= x + \int_t^u v(s, X_s) ds =  x + \int_t^u v(s, f(t,s,x)) ds
\end{align} 
Differentiating the recursive form on the RHS 
w.r.t. $t$ using the total (material) derivative yields:
\begin{align}
\label{eq:pde}
    \partial_t f + (\partial_x f) v(t,x) = 0, \quad f(u,u,x)=x
\end{align}
This is uniquely solved at the true flow map $f$. We can square the left-hand side for a parameterized $f_\theta$ and take an expectation.
$X_t$ is sampled by drawing data $X_1$, noise $X_0$, and computing $X_t = \alpha_t X_0 + \sigma_t X_1$:
\begin{lavenderbox}
\begin{align}
 L 
    &:=
     \mathbb{E}_{X_t}[\| \partial_t f_\theta + (\partial_x f_\theta) v\|^2]
     =
    \mathbb{E}_{X_t}[\| \partial_t f_\theta + (\partial_x f_\theta) \mathbb{E}[\dot X_t | X_t]\|^2].
\end{align}
\end{lavenderbox}
The true map $f$ is the unique minimizer of this loss.
Using $v(t,x)= \mathbb{E}[\dot X_t | X_t=x]$, we can expand,
\begin{align}
    L 
    &= \mathbb{E}_{X_t}[\| \partial_t f_\theta + (\partial_x f_\theta) \dot X_t\|^2
    -
   \|(\partial_x f_\theta)(\dot X_t - \mathbb{E}[\dot X_t | X_t])\|^2]
   \label{eq:L}
\end{align}
For an underlying model $\tilde{f}_\theta$, we can then use the parameterization
\begin{align}
\label{eq:param}
 f_\theta(t,u,x) := x +(u-t) \tilde{f}_\theta(t,u,x)   
\end{align}
This parameterization automatically satisfies the boundary condition $f_\theta(u,u,x)=x$. The parameterization yields two additional properties:
\begin{itemize}
    \item time derivative: $\partial_t f_\theta(t,t,x) = -\tilde{f}_\theta(t,t,x)$
    \item Jacobian: $\partial_x f_\theta(t,t,x)=I$
\end{itemize}
Using these properties and evaluating at $t=u$,
we see that the minimization of \cref{eq:L} reduces to flow matching where $\tilde{f}_\theta(t,t,x)$ is trained to match the velocity:
\begin{align}
   L \big|_{t=u} = \mathbb{E}_{X_t}[\|\tilde{f}_\theta(t,t,X_t) - \dot X_t \|^2],
\end{align}
which reveals that for the true $f$, 
we have that
\begin{align}
 -\partial_t f(t,t,\cdot) = \tilde{f}(t,t,\cdot) = v(t, x) = \mathbb{E}[\dot X_t | X_t = x]   
\end{align}
This motivates replacing the unknown $v$ 
in \cref{eq:L}
with $\text{stopgrad}[\tilde{f}_\theta(t,t,\cdot)]$.
The stopgrad is used under the principle that since the original $v$ did not provide gradient updates for $f$, neither should a term that approximates it.
Let $\text{SG}()$ denote the $\text{stopgrad}()$ operator. The \textbf{\method{}} method follows parameter updates to $\theta$ by differentiating
\begin{lavenderbox}
   \begin{align}
    \label{eq:ourloss}
    L_{\text{sg}}
    := 
\mathbb{E}_{X_t}[\|(\partial_t f_\theta)+ (\partial_ x  f_\theta) \dot X_t \|^2
- 
      \| 
(\partial_x f_\theta)
(\dot X_t 
- \text{SG}[\tilde{f}_\theta]
)
\|^2].
    \end{align}
\end{lavenderbox}

Here, the model $f_\theta$ uses the parameterization from \cref{eq:param}.
The derivatives $\partial_t f_\theta$ and $\partial_x f_\theta$ are evaluated at $(t,u,X_t)$ and $\tilde{f}_\theta$ is evaluated at $(t,t,X_t)$.
The expectation is taken over $X_t$ sampled by drawing data $X_1$, noise $X_0$, and computing $X_t = \alpha_t X_0 + \sigma_t X_1$ and $\dot X_t = \dot \alpha_t X_0 + \dot \sigma_t X_1$.

In practice, the PDE in \cref{eq:pde} must hold for all pairs $t \leq u$.  Let $q(t,u)$ be a joint distribution with support over $ t \leq u$ and with positive probability on $t=u$. We take expectations over time and define
\begin{align*}
 \mathcal{L} = \mathbb{E}_{q(t,u)}[L],
 \quad\mathcal{L}_{\text{sg}} = \mathbb{E}_{q(t,u)}[L_{\text{sg}}]
\end{align*}
We now connect $\mathcal{L}$ and $\mathcal{L}_{\text{sg}}$ formally. \Cref{thm:thm1}
shows that parameter updates of $\mathcal{L}$ and $\mathcal{L}_{\text{sg}}$ are $0$ at the same solutions.

\begin{bluebox}
\begin{theorem}
\label{thm:thm1}
Let $q(t,u)$ be a joint distribution 
over time pairs with support over $t \leq u$ and with positive probability on $t=u$. 
Let the family $\mathcal{\tilde{F}}$ include functions $\tilde{f}$ that are 
continuously differentiable in all arguments. 
Let $X_t = \alpha_t X_0 + \sigma_t X_1$
and $\dot X_t = \dot\alpha_t X_0 + \dot\sigma_t X_1$.
Define $f(t,u,x) := x + (u-t)\tilde{f}(t,u,x)$.
Let expectations be computed over $q(X_0)q(X_1)$.
Let $sg$ stand for stop-gradient.
Define $\mathcal{L} = \mathbb{E}_{q(t,u)}[L]$
and
$\mathcal{L}_{\text{sg}} = \mathbb{E}_{q(t,u)}[L_{\text{sg}}]$.
Then $\tilde{f}^*$ is a stationary point of $\mathcal{L}_{\text{sg}}$ with respect to $\mathcal{\tilde{F}}$ if and only if $\tilde{f}^*$ is a stationary point of $\mathcal{L}$ with respect to $\mathcal{\tilde{F}}$.
\end{theorem}
\end{bluebox}
This is shown in  
\Cref{appsec:flowmap_thm1}.

\paragraph{Intuition.}

The purpose of the theorem is to establish that $\mathcal{L}_{\text{sg}}$ has the same set of solutions as $\mathcal{L}$ despite not having access to $v$;
here, \textit{solution} is defined as a stationary point.
The intuition is that, despite the stopgrad, when $t=u$, $\mathcal{L}_{\text{sg}}$ tries to match the velocity. We show that $\mathcal{L}_{\text{sg}}$ is not at a stationary point when this velocity estimate is inaccurate, so the optimization continues moving and does not become stuck at functions that integrate (``distill") an incorrect velocity. As this match improves so does the match between the parameter updates from  $\mathcal{L}_{\text{sg}}$  and $\mathcal{L}$ at $t \neq u$. The main reason this works is that $\tilde{f}(t,t,\cdot)$ appears in other terms outside of the stopgrad, and those terms tell it where to go. This is crucial and not all stopgrad optimizations  benefit from this property.

\paragraph{Computation.}
Both terms in $\mathcal{L}_{\text{sg}}$ can be computed as expected squared norms of Jacobian-vector products (JVPs),
which  use forward-mode autodifferentiation to avoid explicitly materializing Jacobians, saving memory.
Let the derivatives $(\partial_t f, \partial_u f, \partial_x f)$ be evaluated at $(t,u,x)$.
Using PyTorch notation, 
\begin{align*}
  \text{JVP}[f, (t,u,x), (a,b,c)]  := (\partial_t f)
  \cdot 
  a +
  (\partial_u f) \cdot b + (\partial_x f)\cdot c
\end{align*}
For the first loss term, 
$a=1$, $b=0$, and $c = \dot X_t$.
For the second term:
\begin{align*}
 a=0, \quad b=0, \quad c= \dot X_t + \text{SG}[\partial_t f(t,t,X_t)] = \dot X_t - \text{SG}[\tilde{f}(t,t,X_t)].
\end{align*}
Though we have two distinct JVPs,  we can split the batch  and randomly assign either set of $(a,b,c)$ values to each batch element. 

\begin{bluebox}
\paragraph{Nongradient Flow}
Most analyses in machine learning implicitly assume that optimization performs gradient descent on a scalar loss, so that stationary points coincide with minima. This assumption does not hold here: following the update rules of $\mathcal{L}_{\text{sg}}$ does not correspond to following the gradients of any single scalar objective, as we demonstrate in \cref{appsec:nongradient}. The stopgrad structure breaks the symmetry required for the updates to be expressible as the gradient of a scalar function; such broken symmetry renders the dynamics \textit{non-conservative} \citep{balduzzi2018mechanics}, meaning they do not necessarily minimize any quantity. Thus, in the limit of small step size, the optimization dynamics correspond to a nongradient vector flow, where it is crucial to analyze stationary points rather than minima.
\end{bluebox}

\section{When Can Stopgrads Fix Things?}
\label{sec:when_stopgrads_fail}

The stopgrad in \method{} is not just a computational convenience,
but is essential for correctness. To see why, consider what happens without it.
Meanflow \citep{geng2025mean} and \method{} both replace the unknown conditional velocity $v(t,x) = \E[\dot X_t \mid X_t = x]$ with the sample-level $\dot X_t$. In standard flow matching this substitution is harmless: the loss is linear in $\dot X_t$, so pulling the conditional expectation out of the square only adds a constant, the posterior variance,
\begin{align}
V(t,x) := \Var(\dot X_t \mid X_t = x),
\end{align}
and does not move the optimum. \method{}'s second term compensates for this. In Meanflow without stopgrads, the term $\partial_x \tilde{f}\, v$ makes the integrand \emph{nonlinear} in $\dot X_t$. The $v \to \dot X_t$ replacement, without a compensating term, then introduces cross-terms involving $V > 0$ that shift the gradient away from the truth. By contrast, with the stopgrad, $\dot X_t$ shows up linearly in the Meanflow parameter updates, meaning the expected update is equal to the correct one with $v$. This was also used to show correctness in \cite{sabour2025align}.

\paragraph{Gaussian example.}
We make this concrete in 1D. Let $X_0 \sim \cN(0,1)$ and $X_1 \sim \cN(\mu, \lambda^2)$ with $\lambda > 0$.
Let us use the linear interpolant
$X_t = (1-t)X_0 + t X_1$ with $\dot X_t = X_1 - X_0$. In this setting everything is in closed form (Appendix~\ref{appsec:proofs_mf_nosg}):
the marginal variance is $\sigma_t^2 = (1-t)^2 + t^2\lambda^2$,
the flow map Jacobian is $\partial_x f^* = m(t,u) = \sigma_u/\sigma_t$,
and the posterior variance is $V(t) = \lambda^2 / \sigma_t^2 > 0$.

Consider the Meanflow functional without stopgrads,
\begin{align}
  \label{eq:mf_nosg}
  \cL_{\mathrm{nosg}}[\tilde{f}]
  = \E\Big[\big\|
    \tilde{f} - \dot X_t - (u-t)\big(\partial_t \tilde{f}
    + \partial_x \tilde{f}\, \dot X_t\big)
  \big\|^2\Big].
\end{align}
Evaluating the first variation at the true flow map $\tilde{f}^*$ in the direction $h(t,u,x) = x$ gives (see Appendix~\ref{appsec:proofs_mf_nosg} for the calculation):
\begin{align}
  \delta \cL_{\mathrm{nosg}}[\tilde{f}^*;\, h = x]
  = 2\lambda^2\,\E_{t,u}\!\left[\frac{(u-t)\,m(t,u)}{\sigma_t^2}\right] > 0.
\end{align}
Every factor in the integrand is strictly positive, so the gradient at the truth is nonzero---the true flow map is \emph{not} a stationary point of the no-stopgrad functional.
With the stopgrad, Theorem~\ref{thm:thm1} guarantees the true flow map is the unique stationary point.

\section{Experiments}

\subsection{Gaussian illustration: Meanflow with and without stopgrads}
\label{sec:gaussian_experiment}

As shown in \S\ref{sec:when_stopgrads_fail}, removing stopgrads from Meanflow destroys stationarity at the true flow map. We illustrate this concretely in the 1D Gaussian setting. We parameterize the target distribution as $X_1 \sim \cN(\mu, \lambda^2)$ and study the parameter update field of the Meanflow functional over the parameter space $(\mu, \lambda)$. The model family is well-specified:
\begin{align}
 \tilde f_{\hat\mu, \hat\lambda}
 =
\textrm{exact flow map for } \cN(0, 1) \to \cN(\hat\mu, \hat\lambda^2)
\end{align}
The truth $(\mu^*, \lambda^*)$ lies in the model family. All expectations admit closed-form expressions in the Gaussian case with linear interpolant.

We plot the update fields
in \Cref{fig:gaussian_gradient}.
With the stopgrad, the true flow map is the unique stationary point, so the negative update arrows converge to the truth. Without the stopgrad, the true parameters are not stationary: the 
update field, which is a gradient field, is nonzero at the truth, and the field drifts toward a spurious point.
 Due to the form of the Gaussian family, it happens that $\hat\mu^{\text{spurious}} = \mu^*$ exactly, with the bias entirely along $\hat\lambda$.

\begin{figure}[t]
  \centering
\includegraphics[width=\textwidth]{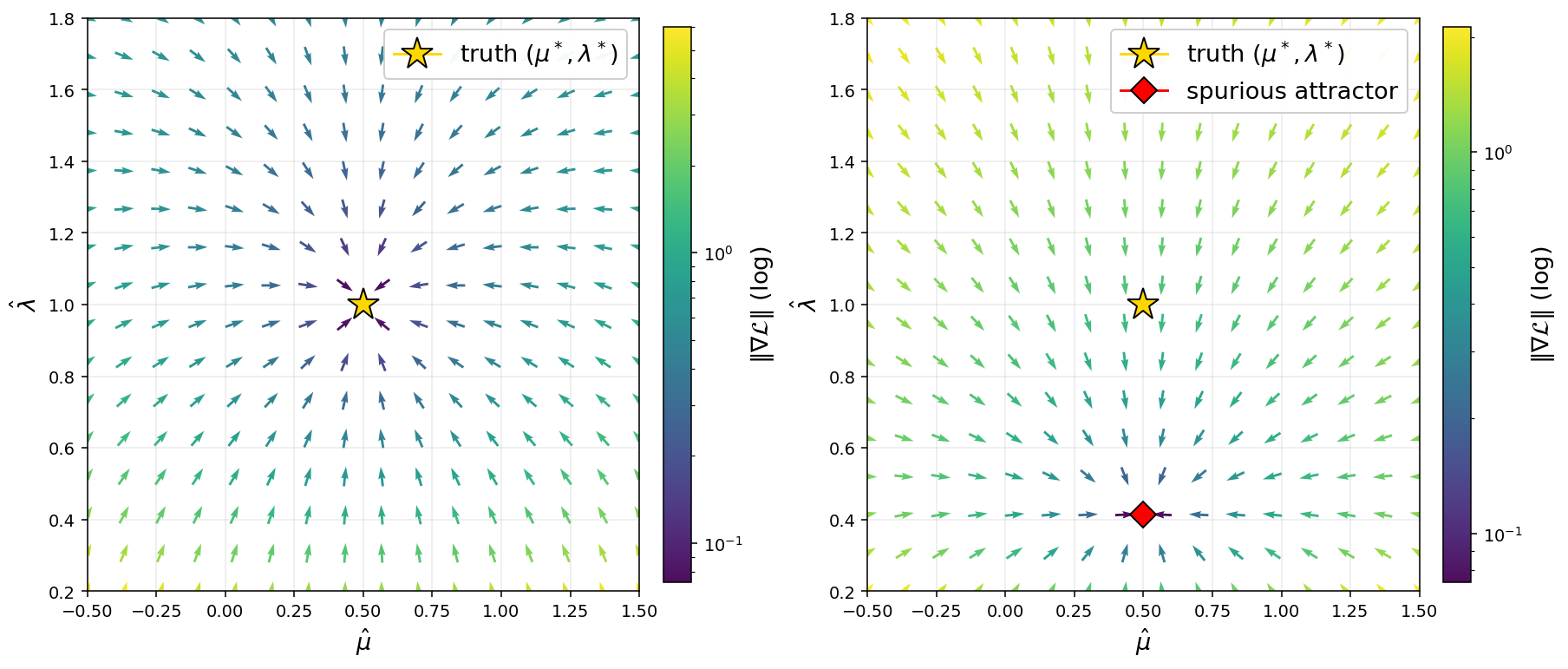}
  \caption{\textbf{Update vector field 
  of the Meanflow functional} over the well-specified Gaussian family, with data $(\mu^*, \lambda^*) = (0.5, 1.0)$ (star, $\star$). 
  We use $\nabla \mathcal{L}$ to denote
  the update field, though it is not a gradient field in the stopgrad case.
  Arrows are unit-normalized; color encodes $\| \nabla \mathcal{L} \|$ on log scale.
  \textbf{Left: with stopgrad}. The truth is the unique stationary point; arrows converge to $\star$ from all directions, and the update magnitude vanishes to machine precision at the truth. \textbf{Right: without stopgrad}. The truth is not stationary; the field flows past $\star$ and converges to a spurious attractor at $(\hat\mu, \hat\lambda) \approx (0.50, 0.41)$ (diamond, $\blacklozenge$). 
  Due to the form of the Gaussian family, it happens that $\hat\mu^{\text{spurious}} = \mu^*$ exactly, with the bias entirely along the $\hat\lambda$ component.}
  \label{fig:gaussian_gradient}
\end{figure}

\subsection{Image Modeling on CIFAR-10}

\paragraph{Architecture.}
We modify the time embedding of the existing diffusion U-Net from \citet{dhariwal2021diffusion} to handle two times. We embed both times $t$ and $u$  with the usual Fourier embeddings, concatenate them, and pass the result through a small feedforward network that outputs a hidden representation for use in the usual U-Net.  We use 128 U-Net channels and channel multipliers set to (1,2,2,2) with attention bools set to (False, False, True, False). 

\paragraph{Training Settings.}
We use dropout $0.1$. We choose the unconditional image generative modeling task (we do not condition on the class label).  We use $\alpha_t=1-t$ and $\sigma_t=t$, with noise at $X_0$ and data at $X_1$. We train for 200{,}000 steps at learning rate $2\times10^{-4}$. The FID is monitored 
every 25{,}000 training steps, and the best is reported.

\paragraph{Losses.}
We benchmark based on
Flow Matching \citep{lipman2022flow} as the reference method for sampling in the $O(100)$ step regime.
For few-step sampling, we benchmark with Meanflow \citep{geng2025mean},
which stopgrads all model derivatives in the loss to avoid backpropagation through differentiation.
We also benchmark against Lagrangian flow map matching \citep{boffi2024flow}
using the stopgrad placement suggested in \cite{boffi2025build}.
We do not compare against the Eulerian self-distillation (ESD) loss, which \cite{boffi2025build} report as unstable for image experiments, nor against the Progressive self-distillation (PSD) loss, which enforces the semigroup identity rather than a transport PDE. We focus here on the PDE-based functionals.
\begin{table}[h]
\begin{graybox}
\centering
\begin{tabular}{lrrrrr}
\toprule
\textbf{Method} & \textbf{1 step} & \textbf{10 steps} & \textbf{50 steps} & \textbf{100 steps} & \textbf{SG theory} \\
\midrule
Flow Matching & >100 & 84.39 & 4.51 & 3.97 & no \\
Meanflow & \textbf{8.67} & 5.26 & 5.12 & 4.07  & no \\
Lagrangian & 26.21& 4.08 & \textbf{3.37} & \textbf{3.20} & no  \\
\method{} & 37.29 & \textbf{4.00} & 3.87 & 3.85 & \textbf{yes}  \\
\bottomrule
\end{tabular}
\end{graybox}
\vspace{0.1cm}
\caption{FID scores versus sampling steps on CIFAR-10, computed from 50{,}000 EMA samples. The FID is monitored every 25{,}000 training steps up to 200{,}000 training steps, and the best is reported. Lower is better; \textbf{bold} indicates best per column. The ``SG theory'' column indicates whether the method has a proven stationary-point guarantee for its stopgrad-based dynamics (i.e., whether the true flow map has been proven to be a stationary point of the optimization).}
\label{tab:cifar_fid}
\end{table}

\paragraph{Results.} We report the Fr\'echet Inception Distance (FID) \citep{heusel2017gans} in \Cref{tab:cifar_fid}. No single method dominates across all step counts: Meanflow is strongest at 1 step, the Lagrangian loss is strongest at 50 and 100 steps, and \method{} attains the best FID at 10 steps.
\method{} is the only method in the comparison whose stationary points under stopgrad-induced dynamics are characterized, while remaining competitive with,  and at 10 steps surpassing, methods that do not provide such guarantees. This suggests that the stationary-point analysis we develop does not come at a cost in empirical performance.

\section{Related work}

\label{sec:flowmap_related_work}
Sampling from continuous-time generative models such as diffusion and flow models requires numerical integration. Each integration step requires a forward pass of a neural network, leading to computational costs and slow sampling. Current approaches to address this cost can be broadly categorized into two types: (1) distilling a pretrained diffusion or flow model into a few-step solver \citep{salimans2022progressive, kim2023consistency, liu2023instaflow}, and (2) learning a few-step solver \citep{zhou2025inductive}. Some approaches in this area allow for distillation as well as training from scratch \citep{song2023consistency,boffi2024flow, boffi2023probability}.

Consistency Models (CMs) \citep{song2023consistency,song2023improved, lu2024simplifying} learn a one-step map from noise to data, either by distilling a pretrained model or by learning from scratch. Distillation requires sampling trajectories from the teacher model. 
To allow for more steps after either training approach, CMs iteratively re-noise the one-step solution back to successively smaller time under the interpolant and then denoise, but this can take the solver off the probability flow.

Consistency trajectory models (CTMs) \citep{kim2023consistency} extend CMs to learn two-time maps using a combination of consistency and adversarial objectives, which requires training an additional discriminator model \citep{goodfellow2014generative}.  CTM and \method{} both target the same mathematical object, the probability flow ODE flow map (i.e., the integral of the ODE), but they learn this map through different means. CTM learns the map by distilling a teacher solver, and the losses for teacher and student involve several nested model evaluations (with data at $x_0$, for $0 \leq s \leq u \leq t \leq 1$, the teacher integrates from $t$ to $u$, then jumps from $u$ to $s$, then from $s$ to $0$; and the student jumps from $t$ to $s$ and then to $0$). The objective depends on a chosen feature-space distance and, in practice, includes DSM and GAN terms that further influence the optimum. Consequently, the CTM loss is sensitive to the quality of the ODE discretization used by the teacher (in practice CTM finds the need to use a 2nd order solver during training) and necessitates the presence of the GAN. 

Inductive Moment Matching  (IMM) \citep{zhou2025inductive} learns a few-step  model via an implicit generative model trained with MMD \citep{smola2006maximum, gretton2012kernel}, where the MMD is estimated with bias within subsets of data. In practice, the authors must use  time-weighting schedules and specific curriculum/inductive procedure to stabilize optimization.  While IMM produces high-quality image samples,
it solves the problem of marginally sampling the data distribution rather than sampling along a probability flow, where the latter is the task studied in this work.

Meanflow \citep{geng2025mean} derives a JVP-based objective for flow maps from the same PDE in \cref{eq:pde}:
\begin{align}
\mathcal{L}_{t,u}^{\text{meanflow}} := \mathbb{E}[\| \tilde{f}_\theta(t,u,X_t) - \dot X_t - (u-t)\text{SG}(\partial_x \tilde{f}_\theta \cdot \dot X_t + \partial_t \tilde{f}_\theta)\|^2].
\end{align}
Here Meanflow is written in forward time with $t$ derivatives rather than reverse time with $u$ derivatives (\cref{appsec:t_vs_u}).
Applying the stopgrad $\text{SG}$ to all model derivatives improves efficiency,  but there are no differentiated loss terms that encourage the model derivatives $\partial_t \tilde{f}_\theta$ and $\partial_x \tilde{f}_\theta$ to move toward the true flow map derivatives. This contrasts with \method{}
where $\text{SG}[\tilde{f}_\theta(t,t,x)]$ is used in place of $\mathbb{E}[\dot X_t | X_t]$, but where another term in the loss trains these two quantities to match.  Finally, between equations (10, 11) in \cite{geng2025mean}, $v$ is replaced with $\dot X_t$ where it appears quadratically, thereby pulling an expectation through a square and missing a resulting trace covariance term \citep{boffi2025build}; interestingly, this causes Meanflow to have wrong stationary points when the stopgrad \textit{is not} used. Meanflow \textit{with} stopgrad does have correct stationary points \citep{sabour2025align}.

Flow Map Matching \citep{boffi2024flow} learns a two-time flow map, enabling mapping along the probability flow in either direction without adversarial training. The Lagrangian loss requires only time derivatives, but relies on an additional invertibility loss, which encourages invertibility via swapping the time arguments:
\begin{align*}
 f_\theta(t, u, f_\theta(u,t,x)) \approx x. 
\end{align*}
Here the model is trained with the second time argument either larger or smaller than the first.
This is straightforward to compute, but gradient steps require evaluating the model and its inverse at each training step. \cite{boffi2025build} introduce stopgrad placements for the LSD, ESD, and PSD loss functionals from \cite{boffi2024flow} that substantially improve image generation performance. Given this empirical success, understanding why these stopgrads help is important: they likely have good theoretical properties, but this remains to be shown.

\textit{Distillation methods}. A complementary line of work approaches flow map learning by \emph{distilling}
the outputs of fully pretrained flow matching velocity models into few-step solvers. Specifically, the unknown $v$ in the flow map identities is taken to be a pretrained network. By contrast, we emphasize training from scratch, avoiding dependence on a teacher model and ensuring that all components of the flow map are learned end-to-end. That said, distillation can be attractive in practice when a pretrained model is already trusted (when $v_\theta$ corresponds to the endpoint distributions and the chosen $\alpha_t$ and $\sigma_t$, or when the objective is weaker---for example, to marginally sample from the approximated data distribution without explicitly solving the probability flow ODE).

\section{Discussion and Limitations}

\paragraph{Identifying functions through PDEs}
Consider a PDE solved by a sought-after mapping $f$, featuring a combination of terms such as the time and space derivatives $\partial_t f, \partial_x f$.
The PDE being solved means that $f$ sets the residual to $0$. Such an $f$ can be found by minimizing a squared error loss built from the residual.
 Which terms should be parameterized by the model and which should be approximated as part of a ground-truth loss target? If the equations can be rewritten in several ways, which yield easier or more challenging objectives? Answering this is applicable to improving training objectives for generative models as well as solving more general PDE-related tasks with neural networks and optimization.

\paragraph{Invertibility}
The loss targets an invertible function at optimum. To simplify training, we explicitly give up knowing the inverse, meaning that we only learn maps in one direction. Luckily, this is the usual scenario for generative modeling. 
For likelihoods, one can still substitute $-\partial_t f_\theta$ for $v_\theta$ in the probability flow ODE \citep{song2021maximum, boffi2023probability}. Thus this method can be seen from the perspective of training a normalizing flow \citep{tabak2010density,tabak2013family,rezende2015variational, papamakarios2021normalizing} without requiring the invertible architecture or inverse-dependent loss.

\paragraph{Architectures.}
\method{}, Flow Map Matching, Simplified Consistency Models, and Meanflow all specify models whose \textit{time-derivatives} equal the target of diffusion model training, but directly adapt architectures meant for diffusion models themselves. Example architectures used in these works  are the UNet from \citet{dhariwal2021diffusion},  the diffusion transformer from
\citet{peebles2023scalable, ma2024sit},  and the EDM architecture from
\cite{karras2022elucidating, karras2024analyzing}. These architectures may thus be suboptimal for the problem at hand, precisely because the target of interest is defined as a function often computed in many diffusion model forward passes (an integral).  In this work, compute limitations did not allow for the thorough exploration of architectures, but the authors believe that rethinking architectures is a convincing direction to improve the quality and training-efficiency of learned flow maps.

\bibliography{iclr2026_conference}

\appendix

\newpage 

\section*{Acknowledgments}

This work was partly supported by the NIH/NHLBI Award R01HL148248, NSF Award 1922658 NRT-HDR: FUTURE Foundations, Translation, and Responsibility for Data Science, NSF CAREER Award 2145542, NSF Award 2404476, ONR N00014-23-1-2634, Optum, and Apple. This work was also supported by IITP with a grant funded by the MSIT of the Republic of Korea in connection with the Global AI Frontier Lab International Collaborative Research. Mark Goldstein would like to thank the Simons Foundation and Flatiron Institute for funding and compute resources that supported this project. Anshuk Uppal would like to thank the Centre for Basic Machine Learning Research in Life Sciences, and Technical University Denmark. The authors thank Eric Vanden-Eijnden and Amirmojtaba Sabour for fruitful discussion on related work and for providing feedback on the theoretical results.

\section{Proofs for Stationary Points}

\subsection{First Variation Definitions \label{appsec:flowmap_first_variation_definition}}

We consider scalar-valued loss functions $\mathcal{L} : \mathcal{F} \to \mathbb{R}$
that map a function $f \in \mathcal{F}$ to a real value.

Define the tangent space $\mathcal{T}_f(\mathcal{F})$ at $f$.
This space contains functions $h \in \mathcal{T}_f(\mathcal{F})$ such that there exists a curve 
indexed by scalar $\epsilon$ such that for each $\epsilon$, $f_\epsilon \in \mathcal{F}$, 
and we have that $f_0=f$ and $(\frac{d}{d\epsilon} f_\epsilon)|_{\epsilon=0}=h$.

The first variation $\delta \mathcal{L}$ of such a functional
$\mathcal{L}$ evaluated at $f \in \mathcal{F}$ in direction $h \in \mathcal{T}_f(\mathcal{F})$ is defined as:
\begin{align}
    \delta \mathcal{L}[f; h] := 
    \Big(\frac{d}{d \epsilon} 
    \mathcal{L}[f + \epsilon h]\Big)_{\epsilon = 0}
\end{align}
We then have that $f^*$ is a stationary point w.r.t. $\mathcal{F}$ if $\delta\mathcal{L}[f^*;h]=0$ for all $h \in \mathcal{T}_{f^*}(\mathcal{F})$.

\subsection{Stopgrad for Functionals \label{appsec:flowmap_functional_stopgrad}}

We define the stopgrad symbol $\text{sg}$ for a functional as follows. Let $\mathcal{O}$ be a functional that maps two functions $f, g$ to a real value. Let $\mathcal{L}[f]$ be a functional that is written in terms of $\mathcal{O}$ with symbol $\text{sg}$ as $\mathcal{L}[f] := \mathcal{O}[f, \text{sg}[f]]$, then we evaluate the following two quantities as follows
\begin{align}
    \mathcal{L}[f] 
    &= \mathcal{O}[f, f]\\
    \delta \mathcal{L}[f ; h]
    &= \delta \mathcal{O}[f, f ; h, 0]
\end{align}
That is, the functional evaluates as usual but in a first variation, we do not perturb terms in $\text{sg}$. This corresponds to the stopgrad or detach() operation used in machine learning code with autodifferentiation.

\newpage

\subsection{First Variation of Original Loss Functional \label{appsec:flowmap_first_variation_L}}

Our functional $\mathcal{L}[\tilde{f}]$ 
acts on functions $\tilde{f}$.
According to the definitions in 
\Cref{appsec:flowmap_first_variation_definition}, we need to compute $\delta \mathcal{L}[\tilde{f}; h] = (\frac{d}{d \epsilon }\mathcal{L}[\tilde{f} + \epsilon h])_{\epsilon=0}$. The functional is:
\begin{align*}
T_1[\tilde{f}] :&=
\| (\partial_t f)_{(t,u,x_t)} + (\partial_x f)_{(t,u,x_t)} \dot x_t
    \|^2_{f = x + (u-t)\tilde{f}}
    \\
T_2[\tilde{f}] :&= 
\| 
(\partial_x f)_{(t,u,x_t)}
    (\dot x_t - \mathbb{E}[\dot x_t | x_t])
    \|^2_{f = x + (u-t)\tilde{f}}
\\
A[\tilde{f}]
&=
T_1[\tilde{f}] - T_2[\tilde{f}]\\
\mathcal{L}[\tilde{f}]
    &= 
    \mathbb{E}_{
    q(t,u),
    q(x_0),q(x_1)
    }\Big[A[\tilde{f}]\Big]
\end{align*}
Lets define the path $f_\epsilon$
by replacing $\tilde{f}$ with $\tilde{f}_\epsilon := \tilde{f} + \epsilon h$. Then:
\begin{align}
 f_\epsilon 
 &:= x + (u-t) \tilde{f}_\epsilon
 = x + (u-t)(\tilde{f} + \epsilon h) 
 = x + (u-t)\tilde{f} + \epsilon (u-t)h   
\end{align}

Then 
\begin{align}
    \frac{d}{d\epsilon} \mathcal{L}[\tilde{f} + \epsilon h] 
    = 
    \mathbb{E}\Big[\frac{d}{d\epsilon} A[\tilde{f}_\epsilon]\Big]
    =
    \mathbb{E}\Big[
    \frac{d}{d\epsilon}T_1[\tilde{f} + \epsilon h]
    -
    \frac{d}{d\epsilon}T_2[\tilde{f} + \epsilon h]
    \Big]
\end{align}
We first compute this derivative and then evaluate it at $\epsilon=0$.

So
\begin{align}
    \partial_t f_\epsilon 
    &= 
    \partial_t \Big[ 
        x + (u-t)\tilde{f} + \epsilon(u-t)h
    \Big]\\
    &= 
    \partial_t  x
    + 
    \partial_t 
    \Big[ 
      (u-t)\tilde{f}
      \Big] 
      + \epsilon
      \partial_t \Big[
      (u-t)h
    \Big]\\
    &= 
    (u-t)\partial_t \tilde{f} 
    - \tilde{f} 
    + 
    \epsilon 
    \Big(
    (u-t)\partial_t h 
    -
  h
    \Big)\\
    &= 
    (u-t)(\partial_t \tilde{f} 
    + \epsilon \partial_t h)
    - (\tilde{f} + \epsilon h)    
\end{align}
and 
\begin{align}
    \partial_x f_\epsilon 
    &=
    \partial_x 
    \Big[ 
    x + (u-t) \tilde{f}
    + \epsilon (u-t) h 
    \Big]
    =
    I + (u-t)\partial_x (\tilde{f} + \epsilon h)    
\end{align} 
and 
\begin{align}
    \frac{d}{d\epsilon} \partial_t f_\epsilon &= 
    \frac{d}{d\epsilon} 
  \Big[ (u-t)(\partial_t \tilde{f} 
    + \epsilon \partial_t h)
    - (\tilde{f} + \epsilon h)    \Big]  \\
&= 
    \frac{d}{d\epsilon} 
  \Big[ (u-t)\partial_t \tilde{f}
  + \epsilon (u-t) \partial_t h
  - \tilde{f} - \epsilon h
   \Big]  \\
&= 
    \frac{d}{d\epsilon} 
    \Big[
     \epsilon (u-t) \partial_t h
    - \epsilon h
    \Big]\\
&= (u-t)\partial_t h - h
\end{align}
and
\begin{align}
    \frac{d}{d\epsilon} \partial_x f_\epsilon
    &=
\frac{d}{d\epsilon}
\Big[
I + (u-t)\partial_x (\tilde{f} + \epsilon h)
\Big]\\   
&=
\frac{d}{d\epsilon}
I 
+ 
\frac{d}{d\epsilon}
 (u-t)\partial_x \tilde{f}
 +
 \frac{d}{d\epsilon}
 (u-t)\partial_x 
 \epsilon h\\
&=
 (u-t)\partial_x h
\end{align}

For the first term, 
\begin{align*}
    T_1[\tilde{f} + \epsilon h]
    =
    \| \partial_t f_\epsilon 
    + (\partial_x f_\epsilon) \dot x_t 
    \|^2 
\end{align*}
Differentiating 
\begin{align}
    \frac{d}{d \epsilon}
        T_1[\tilde{f} + \epsilon h]
        &=
2\Big(\partial_t f_\epsilon 
    + (\partial_x f_\epsilon) \dot x_t \Big)^\top
    \frac{d}{d\epsilon}
\Big(\partial_t f_\epsilon 
    + (\partial_x f_\epsilon) \dot x_t\Big) \\
        &=
2\Big(\partial_t f_\epsilon 
    + (\partial_x f_\epsilon) \dot x_t \Big)^\top
    \Big(
    \underbrace{(u-t)\partial_t h - h}
    + 
    \underbrace{(u-t)\partial_x h} \dot x_t\Big) \\
        &=
2\Big(\underbrace{(u-t)(\partial_t \tilde{f} + \epsilon \partial_t h) - (\tilde{f}+\epsilon h)}
    + (\underbrace{I + (u-t)\partial_x (\tilde{f} + \epsilon h) })\dot x_t \Big)^\top\\
    & \quad 
    \Big(
    \underbrace{(u-t)\partial_t h - h}
    + 
    \underbrace{(u-t)\partial_x h} \dot x_t\Big) 
\end{align}
So 
\begin{align}
    \frac{d}{d\epsilon} T_1[\tilde{f} + \epsilon h] \Big|_{\epsilon = 0}
        &=
2\Big((u-t)(\partial_t \tilde{f}
+
\partial_x \tilde{f} \dot x_t
) - \tilde{f}
    + \dot x_t  \Big)^\top
    \Big(
    (u-t)(\partial_t h + \partial_x h \dot x_t) - h \Big) 
\end{align}
For the second term, 
\begin{align}
    T_2[\tilde{f} + \epsilon h]
    &=
    \| \partial_x f_\epsilon (\dot x_t - v) \|^2
\end{align}
and
\begin{align}
    \frac{d}{d \epsilon}
    T_2 [\tilde{f} + \epsilon h]
    &= 
2\Big[\partial_x f_\epsilon (\dot x_t - v) \Big]^\top     
\frac{d}{d\epsilon} 
\Big[\partial_x f_\epsilon (\dot x_t - v) 
\Big]\\
    &= 
2\Big[\partial_x f_\epsilon (\dot x_t - v) \Big]^\top 
\frac{d}{d\epsilon} 
(\partial_x f_\epsilon) 
(\dot x_t - v)\\
    &= 
2\Big[\Big(    I + (u-t)\partial_x (\tilde{f} + \epsilon h)    \Big) (\dot x_t - v) \Big]^\top 
 (u-t)\partial_x h
(\dot x_t - v)
\end{align} 
So
\begin{align}
    \frac{d}{d \epsilon} T_2[\tilde{f} + \epsilon h]
    \Big|_{\epsilon = 0}
    &=
2\Big[\Big(    I + (u-t)\partial_x \tilde{f}    \Big) (\dot x_t - v) \Big]^\top 
 (u-t)\partial_x h
(\dot x_t - v)
\end{align}
So combining
\begin{align}
\frac{d}{d\epsilon} A \Big|_{\epsilon=0}
    &=
2\Big((u-t)(\partial_t \tilde{f}
+
\partial_x \tilde{f} \dot x_t
) - \tilde{f}
    + \dot x_t  \Big)^\top
    \Big(
    (u-t)(\partial_t h + \partial_x h \dot x_t) - h \Big) \\
& - 
2\Big[\Big(    I + (u-t)\partial_x \tilde{f}    \Big) (\dot x_t - v) \Big]^\top 
 (u-t)\partial_x h
(\dot x_t - v)
\end{align}
At $t=u$, this simplifies
\begin{align}
1[t=u] \frac{d}{d\epsilon}A[\tilde{f}_\epsilon]\Big|_{\epsilon =0}
    =
    2 (\dot x_t - \tilde{f})^\top (-h)
\end{align}
which is the first variation for regression that makes $\tilde{f}$ equal to $E[\dot x_t | x_t]$. Summarizing,
\begin{align}
    \delta \mathcal{L}[\tilde{f}; h]
    &=
    \mathbb{E}\Big[
    2\Big((u-t)(\partial_t \tilde{f}
+
\partial_x \tilde{f} \dot x_t
) - \tilde{f}
    + \dot x_t  \Big)^\top
    \Big(
    (u-t)(\partial_t h + \partial_x h \dot x_t) - h \Big) \\
& - 
2\big[\big(    I + (u-t)\partial_x \tilde{f}  \big) (\dot x_t - v) \big]^\top 
 (u-t)\partial_x h
(\dot x_t - v)
\Big]
\end{align}

\newpage

\subsection{Lemma: velocity matches at a stationary point of original functional \label{lemma:lemma_velocity_mathces_at_stationary_point_of_L}}
\begin{lemma}
Let $\tilde{f}^*$ be a stationary point of $\mathcal{L}$.
Assume that $\tilde{f}^*$ is bounded. 
Assume that
$\tilde{f}^*, v \in C^1$ in arguments $(t,u,x)$ and that all expectations of terms featured in the integrand of $\mathcal{L}$ (i.e., $v, \tilde{f}, \partial_t \tilde{f}, \partial_u \tilde{f}, \partial_x \tilde{f}, \ldots$) are finite.
Then we have that $\tilde{f}^*(t,t,\cdot) = v(t,\cdot)$
    where the velocity $v(t,x)=\mathbb{E}[\dot x_t | x_t]$.
\end{lemma}

\begin{proof}
We proceed by contradiction.
By the premise, we are at a stationary point $\tilde{f}^*$. 
Let $f^* := x + (u-t)\tilde{f}^*$.
By
the definition of stationary point in 
\cref{appsec:flowmap_first_variation_definition}, 
we have that $\delta \mathcal{L}[\tilde{f}^*; h] =0$ for all admissible $h$.
Suppose for the sake of contradiction that at this stationary point, the velocity does not match, meaning
\begin{align}
    -\partial_t f^*(t,t,\cdot) = \tilde{f}^*(t, t, \cdot)
\underbrace{\neq}_{\text{suppose for contradiction}}
    v(t,\cdot)
\end{align}

The proof proceeds by 
picking a direction for which the first variation is nonzero, providing a contradiction to being at a stationary point.
The contradiction (the direction for which the first variation is nonzero) is constructed to arise from assuming that the velocity does not match, meaning that by contradiction the velocity does not match. Specific care is taken to ensure that this direction is admissible, in this case meaning it is a continuous function. 

We name a sequence of functions $g_\eta$ such that there exists $\eta^*$ such that $g_{\eta^*}$ is continuous but yields the nonzero variation when chosen as a direction. To establish this existence under continuity, the dominated convergence theorem is used.

Let us define 
$g(t,u,x) = 1[t=u]\Big(\tilde{f}^*(t,u,x) - v(t,x)\Big)$ and evaluate it at $t=u$ so that 
$g(t,t,x) = \tilde{f}^*(t,t,x) - v(t,x)$. We then define the soft indicator $I_\eta(t,u)$ that goes to $1[t=u]$ as $\eta \to 0$ and define it as: 
\begin{align}
    I_\eta(t,u)
    &=
    1[\eta > 0]2\Big(1- \frac{1}{1 + \exp(-\frac{1}{\eta^2}(t-u)^2)}\Big)
    +
    1[\eta =0]1[t=u]
\end{align}
Using the soft indicator, 
we define a sequence of functions $g_\eta$ so that as $\eta \to 0$ we will have pointwise convergence of 
$g_\eta(t,t,x) \to g(t,t,x)$ 
which also means that $g_\eta(t,u,x)$ for $t \neq u$ goes to $0$. We pick
\begin{align}
    g_\eta(t,u,x) 
    = 
    I_\eta(t,u)
    \Big(\tilde{f}^*(t,t,x) - v(t,x)\Big)
\end{align}
\textbf{Pointwise convergence of g in eta.}
We first establish pointwise convergence of $g_\eta$ to $g$ for all arguments $(t,u,x)$ as $\eta \to 0$ from the right.
\begin{align}
\forall \hat \eta \geq 0, \quad 
    \lim_{\eta \to (\hat \eta)^+}
    g_\eta(t,u,x) 
    =
    g_{\hat \eta}(t,u,x)
\end{align}
To show this,  for any $\hat \eta > 0$, 
use continuity of $g_{\hat \eta}(t,u,x)$ in $\hat \eta$ (product of function without $\eta$ times the soft indicator which is continuous). Then to establish for $\hat \eta = 0$ , we consider two cases $t=u$ and $t \neq u$. For equality:
\begin{align}
    &\lim_{\eta \to 0^+} g_\eta(t,t,x)\\
    &=
    \lim_{\eta \to 0^+}
        I_\eta(t,t)
    \Big(\tilde{f}^*(t,t,x) - v(t,x)\Big)\\
    &=
    \lim_{\eta \to 0^+}
    \Big[
    1[\eta > 0]2\Big(1- \frac{1}{1 + \exp(0)}\Big)
    +
    1[\eta =0]
    \Big]
    \Big(\tilde{f}^*(t,t,x) - v(t,x)\Big)\\
    &=
    \lim_{\eta \to 0^+}
    \Big[
    1[\eta > 0]1
    +
    1[\eta =0]
    \Big]
    \Big(\tilde{f}^*(t,t,x) - v(t,x)\Big)\\
    &=
    \lim_{\eta \to 0^+}
    1[\eta \geq 0]
    \Big(\tilde{f}^*(t,t,x) - v(t,x)\Big)\\
    &= \tilde{f}^* - v
\end{align}
which equals $g_{\eta=0}(t,t,x)$, establishing continuity. Now for $t \neq u$.
$\forall \delta >0$ we must name an $\eta(\delta)$ such that $|g_{\eta(\delta)} - g_0| < \delta$, i.e., $|g_{\eta(\delta)} - 0| < \delta$.
   Assume $|\tilde{f}^* - v| < k$ uniformly in all input values t,x.  
\begin{align*}
    &\lim_{\eta \to 0^+}g_\eta(t,u,x) \\
    &
  =
  \lim_{\eta \to 0^+}
    \Big[   1[\eta > 0]2\Big(1- \frac{1}{1 + \exp(-\frac{1}{\eta^2}(t-u)^2)}\Big)
    +
    1[\eta =0]1[t=u] \Big]
    \Big(\tilde{f}^*(t,t,x) - v(t,x)\Big)\\
       &
  =
  \lim_{\eta \to 0^+}
    \Big[   1[\eta > 0]2\Big(1- \frac{1}{1 + \exp(-\frac{1}{\eta^2}(t-u)^2)}\Big)
     \Big]
    \Big(\tilde{f}^*(t,t,x) - v(t,x)\Big)
\end{align*}
Since we are finding a $\delta$ 
and $\eta(\delta)$ so that 
$|g_{\eta(\delta)} - 0| < \delta$
which means $|g_{\eta(\delta)}| < \delta$, 
this just means we can set $\delta$ to an upper bound on the term we are limiting: let the indicator take on $1$ as when it is 0 we are done.
\begin{align}
    \delta &= 
    \Big| 
 2(1 - \frac{1}{1 + \exp(...)})
    \Big|
k\\
&\iff \frac{\delta}{k} = 2(1 - \frac{1}{1+\exp(-\frac{1}{\eta^2}(t-u)^2)})\\
& \iff \frac{\delta }{2k} = 1 - \frac{1}{1+\exp(-\frac{1}{\eta^2}(t-u)^2)}\\
&\iff 1 - \frac{\delta}{2k} = \frac{1}{1+\exp(-\frac{1}{\eta^2}(t-u)^2)}\\
&\iff \frac{2k}{2k} - \frac{\delta}{2k} = \frac{1}{1+\exp(-\frac{1}{\eta^2}(t-u)^2)}\\
&\iff \frac{2k - \delta}{2k} = \frac{1}{1+\exp(-\frac{1}{\eta^2}(t-u)^2)}\\
&\iff \frac{2k}{2k-\delta} =1+\exp(-\frac{1}{\eta^2}(t-u)^2)\\
&\iff \frac{2k}{2k-\delta}- 1 =\exp(-\frac{1}{\eta^2}(t-u)^2)\\
&\iff \frac{2k}{2k-\delta}- \frac{2k-\delta}{2k-\delta} =\exp(-\frac{1}{\eta^2}(t-u)^2)\\
&\iff \frac{\delta}{2k-\delta} =\exp(-\frac{1}{\eta^2}(t-u)^2)\\
&\iff \log \frac{\delta}{2k-\delta} =-\frac{1}{\eta^2}(t-u)^2\\
&\iff -\log \frac{\delta}{2k-\delta} =\frac{1}{\eta^2}(t-u)^2\\
&\iff -\frac{\log \frac{\delta}{2k-\delta}}{(t-u)^2} =\frac{1}{\eta^2}\\
&\iff -\frac{(t-u)^2}{\log \frac{\delta}{2k-\delta}}= \eta^2
\end{align}
Now note that the soft indicator is strictly $< 1$ 
and that $g_\eta$ 
for fixed $(t,u,x)$ is between $-k$ and $0$ or $0$ and $k$ depending on the sign of $\tilde{f}^*-v$, but never both. So its magnitude is at most $k$. So
\begin{align}
    | g_{\eta(\delta)} - g_0|
    =
    | g_{\eta(\delta)} - 0 | 
    < \delta  < k
\end{align}
This can help us ascertain that the above square root to solve for $\eta$ will be well defined:
\begin{align}
\delta <  k
 & \implies 2k - \delta  > k
\\ & \implies \frac{1}{2k - \delta} < \frac{1}{k}
\\ & \implies 
\frac{\delta }{2k - \delta} < \frac{\delta }{k}
\\ & \implies 
\frac{\delta}{2k - \delta} < 1 
\\ &  \implies 
\log \frac{\delta}{2k - \delta} < \log 1
\\ & \implies 
\log \frac{\delta}{2k - \delta} < 0
\end{align}
meaning 
\begin{align}
    \eta(\delta) = \sqrt{\frac{(t-u)^2}{|\log \frac{\delta}{2k - \delta}|}}
\end{align}
thus establishing convergence of $g_\eta \to g$ as $\eta \to 0^+$ for each $(t,u,x)$ i.e. pointwise convergence.

Now recall the first variation of $\mathcal{L}$ (\cref{appsec:flowmap_first_variation_L}) and consider it as a function of $\eta$:
\begin{align}
   s(\eta) := \delta \mathcal{L}[\tilde{f}^*; g_\eta]
    &=
    \mathbb{E}
    \Bigg[
    2\Big((u-t)(\partial_t \tilde{f}^*) - \tilde{f}^*
    + (I + (u-t)(\partial_x \tilde{f}^*))\dot x_t \Big)^\top\\
     &\quad 
    \Big(
    (u-t)\partial_t g_\eta - g_\eta
    + 
    (u-t)(\partial_x g_\eta)\dot x_t\Big)\\
& -
2\Big[\Big(    I + (u-t)\partial_x \tilde{f}^*    \Big) (\dot x_t - v) \Big]^\top 
 (u-t)\partial_x g_\eta
(\dot x_t - v)\Bigg]
\end{align}
\textbf{Pointwise convergence of integrand in eta.}
Collect the variables $\omega = (t, u, x_0, x_1)$ and recall that $x_t$ and $\dot x_t$ are functions of $(x_0, x_1)$. Define $\phi_\eta(\omega)$ as shorthand for the expectand so that $s(\eta) = \mathbb{E}[\phi_\eta(\omega)]$. Under the boundedness assumptions and noting that $\phi_\eta$ only polynomially combines $g_\eta$  with $(\partial_t \tilde{f}^*, \partial_u \tilde{f}^*, \partial_x \tilde{f}^*, v, \partial_x v, \ldots)$, similar reasoning used to show $g_\eta \to g$ can also be used to establish that $\phi_\eta \to \phi$ pointwise. 

\textbf{Establish upper envelope.}
In addition to pointwise convergence of $\phi_\eta \to \phi$ as $\eta \to 0$ from the right, we need an upper envelope $G(\omega)$. Beyond the assumptions, the only thing needed to show that an upper envelope exists is to control the term $|(u-t)\partial_t I_\eta(t,u)|$. The derivative of the soft indicator appears since $\partial_t g_\eta = \partial_t (I_\eta g) = (\partial_t I_\eta) g + I_\eta \partial_t g$. At $\eta = 0$ we have $I_\eta(t,u) = 1[t=u]$, so $(u-t)\,\partial_t I_\eta(t,u)$ vanishes identically: it is $0$ for $t \neq u$ because $I_0$ is constant, and for $t = u$ because of the $(u-t)$ prefactor.
For $\eta > 0$, define:
\begin{align}
    z := \frac{(t-u)^2}{\eta^2}, \quad \sigma(z) := \frac{1}{1 + \exp(-z)}, 
    \quad \sigma'(z) := \frac{\exp(-z)}{[1+\exp(-z)]^2}
\end{align}
Note that for $\eta > 0$, 
\begin{align}
    \partial_t I_\eta(t,u) = 
    2 \frac{2(t-u)}{\eta^2} \sigma'(\frac{(t-u)^2}{\eta^2})
\end{align}
and so 
\begin{align}
    |(u-t) \partial_t I_\eta(t,u)| = 
    2 (u-t)
    \frac{2(u-t)}{\eta^2} \sigma'(\frac{(t-u)^2}{\eta^2})
    =
    4 z \sigma'(z) \leq \sup_{z \geq 0} 4 z \sigma'(z) \leq C_0 < \infty
\end{align}
Because $r(z):=4z\sigma'(z)$ is continuous and satisfies $r(0)=0$ and $r(z) \to 0$ as $z \to \infty$, it attains a finite maximum.  \textbf{This bound is independent of } $\boldsymbol{\eta}$.
Thus every term containing $(u-t)\partial_t g_\eta$ is uniformly bounded in $\eta$ by a product of a constant (from the bound and $I_\eta \in [0,1)$).
The other quantities in $\phi$ are bounded by assumption.
Thus such a bounding envelope $G(\omega)$ exists.

\textbf{Using dominated convergence.}
First,
\begin{align}
 \lim_{\eta \to 0^+} s(\eta) = \lim_{\eta \to 0^+} \mathbb{E}[\phi_\eta] = \lim_{\eta \to 0^+} p(t=u) \mathbb{E}[\phi_\eta \g t = u] + p(t < u) \mathbb{E}[\phi_\eta \g t < u]
\end{align}
By the pointwise convergence of $\phi_\eta \to \phi$ as $\eta \to 0^+$ and by the envelope, we can compute the limit of the first and second terms separately. Expanding the first term (with $p(t=u)$):
\begin{align*}
\lim_{\eta \to 0^+}  p(t=u) \mathbb{E}[\phi_\eta \g t = u]
    &=
    \lim_{\eta \to 0^+}  p(t=u) \mathbb{E}
    \Bigg[
    2\Big((u-t)(\partial_t \tilde{f}^*) - \tilde{f}^*
    + (I + (u-t)(\partial_x \tilde{f}^*))\dot x_t \Big)^\top\\
     &\quad 
    \Big(
    (u-t)\partial_t g_\eta - g_\eta
    + 
    (u-t)(\partial_x g_\eta)\dot x_t\Big)\\
& -
2\Big[\Big(    I + (u-t)\partial_x \tilde{f}^*    \Big) (\dot x_t - v) \Big]^\top 
 (u-t)\partial_x g_\eta
(\dot x_t - v)\g t = u\Bigg]
\\
&=
\lim_{\eta \to 0^+}  p(t=u)
    \mathbb{E}
    \Bigg[
    2\Big(- \tilde{f}^*
    + \dot x_t \Big)^\top
    \Big(
    - g_\eta \Big) \g t = u\Bigg]
\\
&=  \lim_{\eta \to 0^+}  p(t=u)
    \mathbb{E}
    \Bigg[
    2\Big(- \tilde{f}^*
    + \E[\dot x_t \g x_t] \Big)^\top
    \Big(
    - g_\eta \Big) \g t = u\Bigg]
\\
&=  \lim_{\eta \to 0^+}  p(t=u)
    \mathbb{E}
    \Bigg[
    2\Big(- \tilde{f}^*
    + v \Big)^\top
    \Big(
    - g_\eta \Big) \g t = u\Bigg]
    \\
&=  p(t=u)
    \mathbb{E}
    \Bigg[
    \lim_{\eta \to 0^+}  
    2\Big(- \tilde{f}^*
    + v \Big)^\top
    \Big(
    - g_\eta \Big) \g t = u\Bigg]
\\
&=  p(t=u)
    \mathbb{E}
    \Bigg[
    \lim_{\eta \to 0^+}  
    2 ||- \tilde{f}^*
    + v||_2^2 \g t = u\Bigg]
\end{align*}
This term is greater than zero by the assumption that the velocity does not match at the stationary point and the assumption of positive probability $p(t=u) > 0$.

Expanding the second term (with $p(t<u)$):
\begin{align*}
\lim_{\eta \to 0^+}  p(t<u) \mathbb{E}[\phi_\eta \g t < u]
    &=
    \lim_{\eta \to 0^+}  p(t<u) \mathbb{E}
    \Bigg[
    2\Big((u-t)(\partial_t \tilde{f}^*) - \tilde{f}^*
    + (I + (u-t)(\partial_x \tilde{f}^*))\dot x_t \Big)^\top\\
     &\quad 
    \Big(
    (u-t)\partial_t g_\eta - g_\eta
    + 
    (u-t)(\partial_x g_\eta)\dot x_t\Big)\\
& -
2\Big[\Big(    I + (u-t)\partial_x \tilde{f}^*    \Big) (\dot x_t - v) \Big]^\top 
 (u-t)\partial_x g_\eta
(\dot x_t - v)\g t < u\Bigg]
\\
    &=
    p(t<u) \mathbb{E}
    \Bigg[\lim_{\eta \to 0^+}
    2\Big((u-t)(\partial_t \tilde{f}^*) - \tilde{f}^*
    + (I + (u-t)(\partial_x \tilde{f}^*))\dot x_t \Big)^\top\\
     &\quad 
    \Big(
    (u-t)\partial_t g_\eta - g_\eta
    + 
    (u-t)(\partial_x g_\eta)\dot x_t\Big)\\
& -
2\Big[\Big(    I + (u-t)\partial_x \tilde{f}^*    \Big) (\dot x_t - v) \Big]^\top 
 (u-t)\partial_x g_\eta
(\dot x_t - v)\g t < u\Bigg]
\end{align*}
There's no $\eta$ in the first term in each of the two dot products that make up the expectand, so we can focus on the second term in the dot products, where $t < u$
For the second term in the first dot product:
\begin{align*}
&\lim_{\eta \to 0^+} \Big(
    (u-t)\partial_t g_\eta - g_\eta
    + 
    (u-t)(\partial_x g_\eta)\dot x_t\Big)
\\
&=\lim_{\eta \to 0^+} \Big(
    (u-t)\partial_t g_\eta - I_\eta(t,u) \Big(\tilde{f}^*(t,t,x) - v(t,x)\Big)
    + 
    (u-t)(I_\eta(t,u) \partial_x [\Big(\tilde{f}^*(t,t,x) - v(t,x)\Big)])\dot x_t\Big)
\\
&=\lim_{\eta \to 0^+} (u-t)\partial_t g_\eta
\\
&=\lim_{\eta \to 0^+} (u-t)\partial_t [I_\eta(t,u) \Big(\tilde{f}^*(t,t,x) - v(t,x)\Big)]
\\
&=\lim_{\eta \to 0^+} (u-t) (\partial_t I_\eta) g + I_\eta \partial_t g
\\
&=\lim_{\eta \to 0^+} (u-t) (\partial_t I_\eta) I_\eta(t,u) \Big(\tilde{f}^*(t,t,x) - v(t,x)\Big)
\\
&=\Big(\tilde{f}^*(t,t,x) - v(t,x)\Big) (u-t) \lim_{\eta \to 0^+} (\partial_t I_\eta) I_\eta(t,u)  =0
\end{align*}
The last equality holds because the function and its time derivative both go to zero.

This means the first product in the expectation goes to zero. By a similar argument the second term goes to zero as well.

Putting it all together
\begin{align*}
L:= \lim_{\eta \to 0^+} s(\eta) = \lim_{\eta \to 0^+} \mathbb{E}[\phi_\eta] = \lim_{\eta \to 0^+} p(t=u) \mathbb{E}[\phi_\eta \g t = u] + p(t < u) \mathbb{E}[\phi_\eta \g t < u] > 0
\end{align*}
Resultingly,
\begin{align}
    \exists \eta_0 > 0 \text{ s.t. } 
    \forall \eta^* \text{ s.t. }
    0 < \eta^* < \eta_0 \implies 
    |s(\eta^*) - L| < \epsilon
\end{align}
If we pick $\epsilon = 0.5 L$ 
then
\begin{align}
    |s(\eta^*) - L| < .5 L \implies s(\eta^*) > L - .5 L = .5 L > 0
\end{align}
so $s(\eta^*) > 0$.
But this contradicts being at a stationary point. It cannot be that $\tilde{f}^*(t,t,\cdot) \neq v(t,\cdot)$. Therefore the velocity must match.

\end{proof}

\newpage 
\subsection{Theorem 1 \label{appsec:flowmap_thm1}}

We present a proof about the functional stationary points of the \method{} functional.
We use the definitions of first variation and stationary point from \cref{appsec:flowmap_first_variation_definition}
and the definition of functional stopgrad
from \cref{appsec:flowmap_functional_stopgrad}.

\begin{theorem*}

Let $q(t,u)$ be a joint distribution 
over time pairs with support over $t \leq u$ and with positive probability on $t=u$. 
Let the family $\mathcal{\tilde{F}}$ include functions $\tilde{f}$ that are 
continuously differentiable in all arguments. 
Let $x_t = \alpha_t x_0 + \sigma_t x_1$
and $\dot x_t = \dot\alpha_t x_0 + \dot\sigma_t x_1$.
Define $f(t,u,x) := x + (u-t)\tilde{f}(t,u,x)$.
Let expectations be computed over $q(X_0)q(X_1)$.
Let $sg$ stand for stop-gradient.
Define:
{\small
\begin{align*}
    \mathcal{L}[\tilde f]
    &:=
\mathbb{E}
    \Bigg[ 
    \| (\partial_t f)|_{(t,u,x_t)}+ (\partial_x f)|_{(t,u,x_t)} \dot x_t 
    \|^2 
    -
    \| 
    (\partial_x f)_{(t,u,x_t)}
    (\dot x_t - \mathbb{E}[\dot x_t | x_t]
    )\|^2
    \Bigg]
    \\
    \mathcal{L}^{\text{sg}}[\tilde{f}]
    &:= 
    \mathbb{E}
    \Bigg[ 
    \| (\partial_t f)_{(t,u,x_t)} + (\partial_x f)_{(t,u,x_t)}  \dot x_t 
    \|^2 
    -
    \| 
    (\partial_x f)_{(t,u,x_t)} 
    (\dot x_t + \text{sg}[(\partial_t f)]_{(t,t,x_t)})
    \|^2
    \Bigg]
\end{align*}
}
Then $\tilde{f}^*$ is a stationary point of $\mathcal{L}^{\text{sg}}$ with respect to $\mathcal{\tilde{F}}$ if and only if $\tilde{f}^*$ is a stationary point of $\mathcal{L}$ with respect to $\mathcal{\tilde{F}}$.
\end{theorem*}
\begin{proof} \textbf{Case 1: If $\tilde{f}^*$ is a stationary point of $\mathcal{L}$, then $\tilde{f}^*$ is a stationary point of $\mathcal{L}^{\text{sg}}$.}

    \begin{itemize}
    \item 
    Since $\tilde{f}^*$ 
    is a stationary point, 
    $\delta \mathcal{L}[\tilde{f}^*;\cdot]=0$

    \item 
    by \cref{lemma:lemma_velocity_mathces_at_stationary_point_of_L},
    we have that $\partial_t f^*(t,t,x) = -\tilde{f}^*(t,t,x)= -\mathbb{E}[\dot x_t | x_t]$
    where $f^* = x +(u-t)\tilde{f}^*$
    
    \item $\mathcal{L}^{\text{sg}} = \mathcal{L}$

    \item Since $\mathcal{L}^{\text{sg}} = \mathcal{L}$,
    then $\delta \mathcal{L}^{\text{sg}}[\tilde{f}^*; \cdot] = \delta \mathcal{L}[\tilde{f}^*;\cdot] = 0$
    \end{itemize}

\noindent \textbf{Case 2: If $\tilde{f}^*$ is not a stationary point of $\mathcal{L}$, then $\tilde{f}^*$ is not a stationary point of $\mathcal{L}^{\text{sg}}$.}

Since $\tilde{f}^*$ is not a stationary point of $\mathcal{L}$,
then $\exists h$ that is admissible (continuous) such that $\delta \mathcal{L}[\tilde{f}^*; h] \neq 0$.
Then,
\begin{align}
    \underbrace{\delta \mathcal{L}[\tilde{f}^*;h]}_{\text{LHS}}
    &=
    \underbrace{\mathbb{E}
    \Big[ 
    1[t=u] \ldots 
    \Big]}_{\text{RHS-L}}
    +
       \underbrace{ \mathbb{E}
    \Big[ 
    1[t\neq u] \ldots 
    \Big]}_{\text{RHS-R}}
\end{align}
If the LHS is nonzero, then one of RHS-L or RHS-R is nonzero.
Consider both cases.

\textbf{Case 2a: The RHS-R is nonzero and
RHS-L is zero.} RHS-L being zero means that the velocity matches, which means that RHS-R has the same first variation between $\mathcal{L}$ and $\mathcal{L}^{\text{sg}}$. So they must coincide regarding stationary points.

\textbf{Case 2b: RHS-L is nonzero and RHS-R is either zero or nonzero.}

\textbf{Case 2b-i.}  $\delta \mathcal{L}^{\text{sg}}[\tilde{f}^*;h] \neq 0$
directly holds. This is all we are trying to ensure anyway, so we are done in this case.

\textbf{Case 2b-ii.}
Define the soft indicator, $I_\eta$:
\begin{align}
    I_\eta(t,u)
    &=
    1[\eta > 0]2\Big(1- \frac{1}{1 + \exp(-\frac{1}{\eta^2}(t-u)^2)}\Big)
    +
    1[\eta =0]1[t=u].
\end{align}
Define the direction, $\hat h_\eta$:
\begin{align}
    \hat h_\eta(t,u,x) 
    = 
    I_\eta(t,u)
    h(t,u,x) .
\end{align}
$\hat h_\eta$ is continuously differentiable
    for any $\eta>0$. This is true because it is a product of two functions that are each continuously differentiable  ($h$ is assumed continuously differentiable). Recall the mapping $s(\eta)$ from
      \Cref{lemma:lemma_velocity_mathces_at_stationary_point_of_L},
    that maps $\eta$ to $\delta \mathcal{L}[\tilde{f}^*; \hat h_\eta]$.
    Under the conditions of dominated convergence established in
    \Cref{lemma:lemma_velocity_mathces_at_stationary_point_of_L},
    we know that $\exists \eta^* > 0$
    such that $\hat h_{\eta^*}$
    is a continuously differentiable function for which
    $\delta \mathcal{L}[\tilde{f}^*; \hat h_{\eta^*}]  \neq 0$.
    This must mean that the velocity is not matched.
    But we know that if the velocity does not match, we are not at a stationary point of $\mathcal{L}^{\text{sg}}$ either,
    since $\mathcal{L}^{\text{sg}}$ and $\mathcal{L}$ coincide on penalizing velocity matching on $t=u$.
\end{proof}

\section{Other useful derivations and results}

\subsection{Gradient updates are not optimization of one scalar objective via gradients \label{appsec:nongradient}}

We show here that there exists a data distribution and a model such that the \method{} updates are not the gradient of any single scalar objective.
We illustrate this by considering an  simple 1D setting. The key point is that, even in this restricted case, the update field induced by the stopgrad operator has
non-zero curl and therefore cannot be written as the gradient of any scalar objective
\(J(\theta)\). The form to be differentiated is:
\begin{align}
\begin{split}
  L_{\text{sg}}(\theta)
 & =
  \mathbb{E}_{X_t}\!\left[
    \big\|
      \partial_t f_\theta(t,u,X_t)
      + (\partial_x f_\theta(t,u,X_t))\,\dot X_t
    \big\|^2
  \right]\\
  & \quad -
  \mathbb{E}_{X_t}\!\left[
    \big\|
      (\partial_x f_\theta(t,u,X_t))
      \big(
        \dot X_t - \operatorname{stopgrad}[\tilde f_\theta(t,t,X_t)]
      \big)
    \big\|^2
  \right].
  \label{eq:sgflow-sg-general}
  \end{split}
\end{align}
for $f_\theta(t,u,x) = x + (u-t)\tilde{f}_\theta(t,u,x)$.
Let us work in 1D and fix values of $X_t=x$ and $\dot X_t = d$, a constant.
To accomplish this, we can choose \(X_1\) freely and set \(X_0 = X_1 - d\), so that the interpolation satisfies both \(X_t = x\) and \(\dot X_t = d\). In this case the expectations in
\eqref{eq:sgflow-sg-general} collapse to evaluation at this point (equivalently, think of us approximating with 1 Monte Carlo sample).
Now, consider the parameters \(\theta = (\theta_1,\theta_2)^\top\)
and a single time pair \((t,u)\) such that at the position $X_t=x$, 
\begin{equation}
  \partial_t f_\theta(t,u,x) = 0, \qquad
  \partial_x f_\theta(t,u,x) = \theta_1, \qquad
  \tilde f_\theta(t,t,x) = \theta_2.
\end{equation}
(For a sufficiently expressive model, such local values can be realized; we only need
existence of such a configuration.). Plugging these into \eqref{eq:sgflow-sg-general} and dropping the expectation
(single point), the \method{} functional is:
\begin{equation}
  L_{\text{sg}}
  =
  (\theta_1 d)^2
  -
  \big(
    \theta_1\big(d - \text{stopgrad}[\theta_2]\big)
  \big)^2.
  \label{eq:sgflow-simple-loss}
\end{equation}

Let \(\tilde\nabla\) denote differentiation with the stopgrad applied to
\(\theta_2\) in the second term. Differentiating
\eqref{eq:sgflow-simple-loss} with respect to \(\theta_1\) yields
\begin{align}
  \tilde\nabla_{\theta_1} L_{\text{sg}}
  &= 2\theta_1 d^2
     - 2\theta_1\big(d - \theta_2\big)^2 \\
  &= 2\theta_1\Big(d^2 - (d - \theta_2)^2\Big) \\
  &= 2\theta_1\Big(d^2 - (d^2 - 2d\theta_2 + \theta_2^2)\Big) \\
  &= 2\theta_1\big(2d\theta_2 - \theta_2^2\big) \\
  &= 2\theta_1\theta_2(2d - \theta_2).
\end{align}
Here the stopgrad on \(\theta_2\) only affects the second term and does not
change the first term. For \(\theta_2\), all occurrences appear inside a stopgrad and
the first term does not depend on \(\theta_2\), hence
\begin{equation}
  \tilde\nabla_{\theta_2} L_{\text{sg}} = 0.
\end{equation}

Thus the update field induced by \method{} in this simple example is
\begin{equation}
  g(\theta_1,\theta_2)
  := \big(g_1(\theta_1,\theta_2), g_2(\theta_1,\theta_2)\big)
  =
  \big(2\theta_1\theta_2(2d - \theta_2),\ 0\big).
\end{equation}

\medskip
\noindent
If this update were the gradient of some scalar objective
\(J(\theta_1,\theta_2)\), then the mixed partial derivatives would commute:
\begin{equation}
  g_1 = \partial_{\theta_1} J, \qquad
  g_2 = \partial_{\theta_2} J
  \quad\Rightarrow\quad
  \partial_{\theta_2} g_1
  =
  \partial_{\theta_1} g_2.
\end{equation}
However,
\begin{align}
  \frac{\partial g_1}{\partial \theta_2}
  &= \frac{\partial}{\partial \theta_2}
     \Big(2\theta_1\theta_2(2d - \theta_2)\Big)
   = 2\theta_1\big(2d - \theta_2 - \theta_2\big)
   = 4\theta_1(d - \theta_2), \\
  \frac{\partial g_2}{\partial \theta_1}
  &= 0,
\end{align}
and hence the curl of the update field is
\begin{equation}
  \frac{\partial g_1}{\partial \theta_2}
  -
  \frac{\partial g_2}{\partial \theta_1}
  = 4\theta_1(d - \theta_2),
\end{equation}
which is non-zero for generic \(\theta\) (for example, whenever
\(\theta_1 \neq 0\) and \(\theta_2 \neq d\)). Therefore \(g\) is a smooth
vector field with non-zero curl and \emph{cannot} be written as the gradient of any
scalar objective \(J(\theta_1,\theta_2)\) in general.
Put differently, a constraint on the relationship between the parameters of the model and the value of the chosen datapoint is necessary to ensure zero curl.

This 1D example is a specific instantiation of \method{} 
\eqref{eq:sgflow-sg-general} with a simple model and a single training
point. It shows that, once we introduce the stopgrad on
\(\tilde f_\theta(t,t,x)\), the resulting optimization dynamics are in
general \emph{non-conservative}. The stopgrad structure breaks the symmetry required for the updates to be the gradient of a
single scalar function. In this sense, \method{} is formally a (two-player)
\emph{game} rather than standard gradient descent on one potential function---or, if one prefers, a nongradient vector flow.

\subsection{Gaussian counterexample: Meanflow without stopgrads}
\label{appsec:proofs_mf_nosg}

We compute the first variation of the no-stopgrad Meanflow functional $\cL_{\mathrm{nosg}}$ at the true flow map for the 1D Gaussian case
$X_0 \sim \cN(0,1)$, $X_1 \sim \cN(\mu, \lambda^2)$,
$X_t = (1-t)X_0 + tX_1$, $\dot X_t = X_1 - X_0$.

\paragraph{Closed-form Gaussian quantities.}
\begin{itemize}
  \item Marginal variance:
    $\sigma_t^2 = (1-t)^2 + t^2 \lambda^2$.
  \item Velocity:
    $v(t,x) = \alpha(t)(x - t\mu) + \mu$
    where $\alpha(t) = \frac{t\lambda^2 - (1-t)}{\sigma_t^2}$.
  \item Flow map:
    $f^*(t,u,x) = m(t,u)(x - t\mu) + u\mu$
    where $m(t,u) = \sigma_u / \sigma_t$.
  \item Jacobian:
    $\partial_x f^*(t,u,x) = m(t,u)$.
\end{itemize}

\paragraph{Posterior variance.}
We compute $V(t) := \Var(\dot X_t \mid X_t)$.
Since $\dot X_t = X_1 - X_0$ and $X_t = (1-t)X_0 + tX_1$
with $X_0 \perp X_1$:
\begin{align}
  \Var(\dot X_t) &= 1 + \lambda^2, \\
  \Cov(\dot X_t, X_t) &= t\lambda^2 - (1-t), \\
  \Var(X_t) &= \sigma_t^2.
\end{align}
By the Gaussian conditional variance formula:
\begin{align}
  V(t)
  = \Var(\dot X_t) - \frac{\Cov(\dot X_t, X_t)^2}{\Var(X_t)}
  = (1 + \lambda^2) - \frac{(t\lambda^2 - (1-t))^2}{\sigma_t^2}
  = \frac{\lambda^2}{\sigma_t^2}.
\end{align}
(The last equality follows from expanding:
$(1+\lambda^2)\sigma_t^2 - (t\lambda^2 - (1-t))^2 = \lambda^2$.)
In particular, $V(t) > 0$ for all $t \in [0,1]$ when $\lambda > 0$.

\paragraph{First variation of $\cL_{\mathrm{nosg}}$ at the true flow map.}
Let $\tilde{f}^\epsilon = \tilde{f} + \epsilon h$.
The quantity inside the norm of $\cL_{\mathrm{nosg}}$ is
$R(\epsilon) = \tilde{f}^\epsilon - \dot X_t
  - (u-t)(\partial_t \tilde{f}^\epsilon
  + \partial_x \tilde{f}^\epsilon\, \dot X_t)$,
with derivative
$R'(\epsilon) = h - (u-t)(\partial_t h + \partial_x h\, \dot X_t)$.
The first variation is $\delta \cL = 2\,\E[R(0) \cdot R'(0)]$.

Write $\dot X_t = v + \xi$ where $\xi = \dot X_t - v$ satisfies $\E[\xi \mid X_t] = 0$.
At the true flow map, the transport equation kills the $v$-part of $R(0)$,
leaving $R(0) = -\partial_x f^*\, \xi$
where $\partial_x f^* = 1 + (u-t)\partial_x \tilde{f}^*$.
Similarly, $R'(0) = \phi - (u-t)\partial_x h\, \xi$
where $\phi = h - (u-t)(\partial_t h + \partial_x h\, v)$.

Computing $\E[R(0)\cdot R'(0)]$: the cross-term $\E[\partial_x f^*\, \xi\, \phi] = 0$
by the tower property (condition on $(t,u,X_t)$; $\E[\xi \mid X_t]=0$).
The remaining term gives
\begin{align}
  \delta \cL_{\mathrm{nosg}}[\tilde{f}^*; h]
  = 2\,\E\big[(u-t)\, \partial_x f^*\, V(t, X_t)\, \partial_x h\big].
\end{align}

\paragraph{Evaluating in the Gaussian case.}
Substituting $\partial_x f^* = m(t,u)$ and $V(t) = \lambda^2/\sigma_t^2$,
and choosing $h(t,u,x) = x$ (so $\partial_x h = 1$):
\begin{align}
  \delta \cL_{\mathrm{nosg}}[\tilde{f}^*;\, h = x]
  = 2\lambda^2\,\E_{t,u}\!\left[\frac{(u-t)\,m(t,u)}{\sigma_t^2}\right].
\end{align}
Since $\lambda > 0$, $(u-t) > 0$ for $t < u$,
$m(t,u) = \sigma_u/\sigma_t > 0$,
and $\sigma_t > 0$,
every factor in the integrand is strictly positive.
Therefore
$\delta \cL_{\mathrm{nosg}}[\tilde{f}^*;\, h = x] > 0$:
the first variation is nonzero, so $\tilde{f}^*$ is \emph{not} a stationary point
of $\cL_{\mathrm{nosg}}$.

\subsection{Differentiating w.r.t.\ $t$ versus $u$}
\label{appsec:t_vs_u}
In this work, $f(t,u,x)$ maps forward in time from $X_t=x$ to $X_u$ along the flow $dX_s = v(s,X_s)$. In Meanflow \citep{geng2025mean},
the map is written in reverse time. The transport PDE can be obtained by differentiating the forward-time flow map with respect to $t$, or the reverse-time flow map with respect to $u$. In both cases, let $t \leq u$:
\begin{align}
    f_{\text{fwd}}(t,u,x) = x + \int_t^u v(s, X_s)ds
      = x + \int_t^u v(s, f_{\text{fwd}}(t,s,x))ds
\end{align}
On the other hand, if moving in reverse time:
\begin{align}
    f_{\text{rev}}(t,u,x) = x + \int_u^t v(s, X_s)ds
        = x - \int_t^u v(s, f_{\text{rev}}(s,u,x)) ds
\end{align}
Differentiating via the Leibniz rule yields:
\begin{align}
    \partial_t f_{\text{fwd}}(t,u,x)
    +
    (\partial_x f_{\text{fwd}})(t,u,x)
    v(t,x) &= 0\\
    \partial_u f_{\text{rev}}(t,u,x)
    +
    (\partial_x f_{\text{rev}})(t,u,x)
    v(u,x) &= 0
\end{align}

\end{document}